\documentclass{article}

     \usepackage[final]{neurips_2024}

\usepackage{microtype}
\usepackage{graphicx}
\usepackage{booktabs} 

\usepackage[utf8]{inputenc} 
\usepackage[T1]{fontenc}    
\usepackage{hyperref}       
\usepackage{url}            
\usepackage{booktabs}       
\usepackage{amsfonts}       
\usepackage{nicefrac}       
\usepackage{microtype}      
\usepackage{xcolor}         

\usepackage{amsmath}
\usepackage{amssymb}
\usepackage{mathtools}
\usepackage{amsthm}\usepackage{multirow}
\usepackage{wrapfig}
\usepackage{graphicx}
\usepackage{subcaption}
\usepackage{enumitem}

\usepackage[capitalize,noabbrev]{cleveref}
\newtheorem{definition}{Definition}
\newtheorem{assumption}{Assumption}
\newtheorem{theorem}{Theorem}

\newtheorem{property}{Property}

\usepackage{centernot}
\newcommand{\CI}{\mathrel{\perp\mspace{-10mu}\perp}}
\newcommand{\nCI}{\centernot{\CI}}

\newcommand{\mie}{\operatorname{MIE}}
\newcommand{\amie}{\operatorname{AMIE}}

\usepackage[textsize=tiny]{todonotes}

\title{Linking Model Intervention to Causal Interpretation in Model Explanation}

\author{Debo Cheng~\textsuperscript{1}\footnotemark[1]~, Ziqi Xu~\textsuperscript{2}\footnotemark[1]~, Jiuyong Li~\textsuperscript{1}, Lin Liu~\textsuperscript{1}, Kui Yu~\textsuperscript{3}, Thuc Duy Le~\textsuperscript{1} \& Jixue Liu~\textsuperscript{1}  \\
	University of South Australia~\textsuperscript{1}~~~~~
	RMIT University~\textsuperscript{2}~~~~~
	Hefei University of Technology~\textsuperscript{3}\\
	\texttt{\{firstname.lastname\}@unisa.edu.au}~\textsuperscript{1} \\
	\texttt{ziqi.xu@rmit.edu.au}~\textsuperscript{2}~~~~~
	\texttt{yukui@hfut.edu.cn}~\textsuperscript{3} \\
}

\begin{document}

\maketitle
\renewcommand{\thefootnote}{\fnsymbol{footnote}} 
\footnotetext[1]{These authors contributed equally to this work.}

\begin{abstract}
Intervention intuition is often used in model explanation where the intervention effect of a feature on the outcome is quantified by the difference of a model prediction when the feature value is changed from the current value to the baseline value. Such a model intervention effect of a feature is inherently association. In this paper, we will study the conditions when an intuitive model intervention effect has a causal interpretation, i.e., when it indicates whether a feature is a direct cause of the outcome. This work links the model intervention effect to the causal interpretation of a model. Such an interpretation capability is important since it indicates whether a machine learning model is trustworthy to domain experts. The conditions also reveal the limitations of using a model intervention effect for causal interpretation in an environment with unobserved features. Experiments on semi-synthetic datasets have been conducted to validate theorems and show the potential for using the model intervention effect for model interpretation.   


\end{abstract}
	
\section{Introduction}
\label{sec:intro}
Artificial intelligence (AI) models are increasingly used to support decision-making in critical areas such as medicine, law and economics~\cite{bishop2006pattern,jordan2015machine}. As a result, beyond model accuracy, it is essential for decision-makers to understand how a model makes its predictions in order to gain trust in the model~\cite{MurdochPNASReview2021,MILLER20191,Review-InterpretableML2020}. 
 	
A major goal of model explanation is to understand the influence of a feature on the predictions of a model. Intervention intuition is often used for such attribution. When a feature is considered as a cause of the outcome, the difference of the outcome between the presence of the cause and the absence of the cause is used for the contribution of the feature to the model prediction. We call such an attribution the model intervention effect since it indicates the effect of a feature on the outcome when the feature is set to 1 (presence) and 0 (absence) respectively when all other features are kept constant.
Many model explanation methods use model intervention. For example, the prediction difference analysis understands image classification by using an image with a feature excluded (setting its value to the baseline) to quantify the feature's contribution to a prediction~\cite{ZintgrafCAW17}. A similar approach is used for understanding decisions with respect to a representation in NLP tasks by erasing the representation (setting its value to the baseline value) for the feature contribution estimation~\cite{DLiMJ16a}. Causal Concept Effect (CaCE) understands an image classifier by comparing the predictions with and without a concept (which can be considered as a super-pixel or super-feature representing a visual concept)~\cite{goyal2019explaining}. Textual Representation-based Average Treatment Effect (TReATE) extends (CaCE) to estimate the concept effect for a text classification model in terms of concepts~\cite{Feder-CausaLM-2021}. 

 

The model intervention effect is inherently association-based and does not indicate causality since they are estimated using data and a machine learning model. Causality is rooted in real-world phenomena. For example, in an experiment, if a change in a variable leads to a change in the outcome when all other variables are fixed constant in the experiment, the variable is a cause of the outcome. 
In data, it is easy to find the difference in the predicted outcomes when a feature $X$ is assigned to different values with all other features given constant. Such a difference may not reveal that the change in the outcome in the real-world is caused by the feature. In the famous example, when changing the ice cream sales (i.e., the feature $X$) from low to high in data, we can obtain the number of drownings (i.e., the outcome) increase. However, the model intervention effect of ice cream sales on drowns is non-zero, but in reality, ice cream sales are not the cause of drownings. Hence, a model intervention effect obtained in data does not, in general, indicate causality in real life.

It is desirable to link the model intervention effect to causality. If the model intervention effect indicates the direct causal effect of a feature on the outcome in real life, then through the feature attribution by the model intervention effect, domain experts will know whether a model is trustworthy by the consistency between known direct causes and direct causes perceived by the model. Consistency means that the underlying logic used for classification is sound in the view of domain knowledge and adds evidence for the validity of a model.    

    

We distinguish our work from other causal interpretation works at the top level here, and refer readers to Section~\ref{sec_summary} and Appendix~\ref{sec:rel} for details. 
Causality-based XAI methods are broadly classified into two types: Causal effect estimation based~\cite{goyal2019explaining,Feder-CausaLM-2021,pmlr-v108-janzing20a,pmlr-v162-jung22a} and counterfactual causal interpretation~\cite{chou2022counterfactuals,guidotti2022counterfactual}. Our work is related to the first type. Most works of the first type use strong assumptions, such as known causal graphs (or causal ordering of variables), and such strong assumptions limit their applications. We have relaxed the assumption and only assume that the outcome is not a cause of any other variables. No work except~\cite{pmlr-v162-jung22a} considers latent variables. However, the work in~\cite{pmlr-v162-jung22a} deals with limited types of unobserved variables for model explanation (see Section~\ref{sec_summary} for details). Importantly, the objective of our paper is not to propose another measure/method for model explanation, but to establish a link between the commonly used model intervention effect to causality. This work makes the following contributions.

\begin{enumerate}
\item The work links the model intervention effect to causality and identifies the conditions when the causal interpretation of the model intervention effect is warranted. The theoretical results assist users in using the model intervention effect for causal explanations and also reveal the risks in such explanations when there exist unobserved variables. 


\item Extensive experiments are conducted to validate the correctness of the theorems and demonstrate the usefulness of the conditions (or linkage) identified in the model explanation. 



\end{enumerate}
		
\section{Preliminaries and problem definition}
\label{subsec:interpretationModel}
In this paper, uppercase letters denote variables, while their corresponding values are represented by lowercase letters. Boldfaced uppercase and lowercase letters are used to represent sets of variables and their values, respectively. 

With graphical causal modelling~\cite{spirtes2000causation,pearl2009causality,hernan2010causal}, a causal directed acyclic graph (DAG) is commonly used. A DAG, represented as $\mathcal{G}=(\mathbf{V}, \mathbf{E})$, consists of a set of nodes $\mathbf{V}$ (representing variables) and a set of directed edges $\mathbf{E}$ and has no directed cycles. Within $\mathcal{G}$, a path is a sequence of connected edges between two nodes. $\mathcal{G}$ is known to be a \emph{causal DAG} if a directed edge indicates a causal relationship. For example, $X_i\rightarrow X_j$ represents that $X_i$ is a direct cause (or parent) of $X_j$, and $X_j$ is a direct effect (or child) of $X_i$. In a DAG, a collider node $X_i$ is a node that has multiple edges pointing towards it. More definitions regarding graphical causal modelling such as Markov condition, causal sufficiency and $d$-separation can be found in Appendix~\ref{Appendix:A}. 



In the context of causal inference, an intervention refers to the act of manipulating one variable to observe its effect on another variable~\cite{pearl2009causality} in an experiment. and it can be represented by the ``do'' operator~\cite{pearl2018book}, denoted as $do(X = x)$, where $X$ is the variable that is intervened, and $x$ is the value to which $X$ is set by the intervention. With some strict conditions, the effect of an intervention can be estimated in observational data with do-calculus with a causal graph~\cite{pearl2009causality}. To differentiate it from our model intervention, we refer to the intervention in causal inference as $do$-intervention.     

In our problem setting, $Y=f(\mathbf{X})$ where $\mathbf{X}$ is a set of features, represented as $\mathbf{X} = \{X_1, X_2, \dots, X_m\}$, is the set of features, 
$Y$ is the predicted outcome, and $\mathbf{x} = (x_1, x_2, \ldots, x_m)$ denotes an individual sample of $\mathbf{X}$. 
%
In the following discussion,  we assume $X_i\in \{X_1, X_2, \dots, X_m\}$ and $Y$ are binary. However, the conclusions in this paper apply to situations involving continuous outcomes, and when $Y$ is continuous, $E(y | \mathbf{X=x})$ replaces the probability $P(y | \mathbf{X=x})$ found in the binary formulae. We use $y$ as the shorthand of $Y=1$ and $x_i$ to represent $X_i=x_i$, respectively, wherever the context allows for clear interpretation.

	

We use an intervention intuition to estimate the importance of a feature to model prediction, i.e., the feature value is set to different levels in a model to explore its effect on the model prediction. 

\begin{definition}[Model Intervention Effects (MIE) and the average MIE (AMIE)]
\label{def:amie}
	Given a model $Y = f(\mathbf{X})$ and an instance $\mathbf{x}$, the Model Intervention Effect of feature $X_i$ on the outcome predicted $Y$ with respect to $\mathbf{x}$ is defined as $\mie(X_i, Y | \mathbf{X'=x'}) = P(y | X_i=1, \mathbf{X'=x'}) - P(y | X_i=0, \mathbf{X'=x'})$ where the probabilities are estimated by the model $f(\mathbf{X})$ and $\{X_i \cup \mathbf{X'} = \mathbf{X}\}$. For all instances $\mathbf{x}$ in a dataset that is sampled from the same population as the dataset for training the model $Y = f(\mathbf{X})$, the Average Model Intervention Effect is $\amie(X_i, Y) = \mathbb{E}_{\mathbf{x}} (\mie(X_i, Y | \mathbf{X'=x'}))$.
\end{definition}

MIE quantifies a change of $Y$ when $X_i$ is changed from its baseline ($X_i = 0$) to ($X_i = 1$) while the values of all other features keep what they are. AMIE is the expectation of the change over a dataset which can be the training dataset or a test dataset.  

In general, AMIE does not represent causality. AMIE is an estimate in data (and model) without considering a causal graph and is not a $do$-intervention in causal inference. It is desirable that a feature importance is linked to causal interpretation, i.e., a non-zero AMIE represents a change of $Y$ due to a change of $X_i$ in the real-world system. Therefore, users can evaluate the trustworthiness of a model using feature attribution.  

We will study the conditions when AMIE estimated in the model (and data) represents causality in the physical system. 
 
	
\section{Linking AMIE with causality}
\label{sec:MIE}
In this section, we establish the link between AMIE and the causal mechanism of the underlying system. All proofs are included in Appendix~\ref{sec:MIE_app} due to page limitations. 
	
\subsection{When there are no unobserved variables}
We begin with an ideal case where all variables in the system are observed.  We aim to link an AMIE in a machine learning model (and data) to a real-world phenomenon. To make such a link, we will need a few assumptions about the data and model.   
\begin{assumption}[The faithfulness assumption]
	\label{assump-Data-graph-faithfulness}
	The data is generated from a causal DAG $\mathcal{G}$, and both the data and $\mathcal{G}$ of $\{\mathbf{X} \cup Y\}$ are faithfulness to each other. 
\end{assumption}
\begin{assumption}[The computational compatibility assumption]
	\label{assump-Data-model-faithfulness}
	The predictive model $Y=f(\mathbf{X})$ is built on the data, and the model $f(\mathbf{X})$ represents the conditional probabilities $P(Y|\mathbf{X})$ in data precisely. 
\end{assumption}

The above assumptions present two levels of faithfulness. Firstly, the causal mechanism, i.e., the causal DAG, and data are faithful to each other with all conditional dependencies in the data and the causal mechanism with each other, i.e., every conditional independence relation in the data existing in the causal mechanism, and vice versa. Secondly, the machine learning model is consistent with the data. All conditional dependencies $P(Y | \mathbf{X})$ in the data and the machine learning model are consistent, i.e., the conditional probabilities of $P(Y| \mathbf{X})$ can be derived from the model precisely. So the model contains information on the causal mechanism and feature attribution of a model is possible to be linked to the mechanism. 

\begin{theorem}[Linking AMIEs and direct causes of $Y$ without the presence of unobserved variables]
	\label{theorem_DirectCause-AMIEs01}
	Suppose that Assumptions~\ref{assump-Data-graph-faithfulness} and~\ref{assump-Data-model-faithfulness} hold, and there are no unobserved variables. If $\mathbf{X}$ includes no descendant variables of $Y$, then $\forall X_i\in\mathbf{X}$, $\amie(X_i, Y) \ne 0$ if and only if $X_i$ is a direct cause of $Y$. Equivalently, $\amie(X_i, Y) = 0$ if and only if $X_i$ is not a direct cause of $Y$.
\end{theorem}
	
Theorem~\ref{theorem_DirectCause-AMIEs01} shows the conditions when the AMIE of a feature provides a causal attribution of the feature when the conditions in Theorem~\ref{theorem_DirectCause-AMIEs01} are met. The exclusion of descendant variables of $Y$ from $\mathbf{X}$ means that all post-outcome variables are excluded from $\mathbf{X}$. Domain experts know the distinction between pre-outcome variables and post-outcome variables very well. This requirement is much weaker than that made by existing causal interpretation methods which assume the knowledge of a complete DAG or the causal ordering of all pairs of variables. Based on our assumption, we only need to know $m$ causally ordered pairs of variables between each variable in $\mathbf{X}$ and $Y$ where $m$ is the number of features in $\mathbf{X}$. On the contrary, knowing a complete DAG means we need to know $m^2 /2$ causally ordered pairs of variables. Furthermore, knowing the causally ordered pairs between a variable in $\mathbf{X}$ and $Y$ is easier than knowing the causally ordered pairs of variables in $\mathbf{X}$. For example, let $\mathbf{X}$ represent genes and $Y$ a protein product. Biologists would know gene activities cause protein productivity based on general biology knowledge. However, knowing the causal relationships among genes is much harder. Many causal relationships among genes are unknown.

Let us consider the implications of the theorem in two scenarios.

In the first scenario, there is a real underlying mechanism, such as a gene expression process where $\mathbf{X}$ represent gene expression and $Y$ a gene product, e.g., a protein due to the gene expression, and a machine learning model $Y=f(\mathbf{X})$ is expected to model the gene-protein relationships. Assume two black box models have been built to predict protein production, and they have the same test accuracy, but different sets of features with non-zero AMIEs (called perceived direct causes). Domain experts can inspect the perceived direct causes and trust the one that is more consistent with their domain understanding since the model is more faithful to the underlying mechanism.        

In the second scenario, a machine learning model $Y=f(\mathbf{X})$ is expected to model the classification or labelling mechanism, rather than the working of the underlying mechanism of the real system or process. For example, in the context of image classification, the inputs (features) are assumed as the causes of the output (class label). In this case, no $Y$'s descendants are included in $\mathbf{X}$. Assume that there is a classification (labelling) mechanism that determines how instances are labelled by feature values. We also assume that the data with true class labels are generated from the classification mechanism and that the machine learning model $Y = f(X)$ has learned the conditional probabilities $P(Y | \mathbf{X})$ accurately. So the features with non-zero AMIEs represent direct causes of the class label in the labelling mechanism. The trustworthiness of the classifier can be judged by the perceived direct causes using AMIEs. For example, for an image classifier, if the perceived direct causes contain mostly only style or background information of an image then the classifier is not trustworthy. 


Causal interpretation is possible when a closed world assumption is taken. In the closed world assumption, that data is all we have and hence there will be no unobserved direct causes. Also, $Y$ is the functional outcome of $\mathbf{X}$ and hence no descendant variables of $Y$ are included in $\mathbf{X}$. Therefore, AMIE reveals the underlying mechanism {direct causes of $Y$} used by a machine learning model. 

\subsection{When there are unobserved variables}
\begin{wrapfigure}{r}{0.55\textwidth}
	\vspace{-1cm}
	\centering
	\begin{subfigure}[c]{0.16\textwidth}
		\centering
		\includegraphics[scale=0.3]{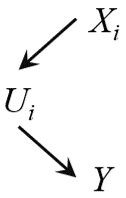}
		\caption{$X_i$ is an acting cause of the unobserved direct cause $U_i$ of $Y$.}
		\label{pic:intro1}
	\end{subfigure}
	\hspace{0.01\textwidth}
	\begin{subfigure}[c]{0.157\textwidth}
		\centering
		\includegraphics[scale=0.3]{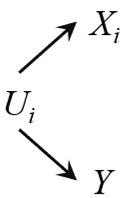}
		\caption{$X_i$ is a proxy of the unobserved direct cause $U_i$ of $Y$.}
		\label{pic:intro2}
	\end{subfigure}
	\hspace{0.01\textwidth}
	\begin{subfigure}[c]{0.13\textwidth}
		\centering
		\includegraphics[scale=0.3]{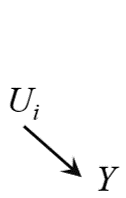}
		\caption{$U_i$ is an independent latent direct cause of $Y$.}
		\label{pic:intro3}
	\end{subfigure}
	\caption{The DAG representing the data generation mechanism assumed in this paper.}
	\label{pic:intro}
	\vspace{-0.5cm}
\end{wrapfigure}

Unobserved variables prevail in real-world applications, and they may be unknown or unmeasured. These unobserved variables introduce additional complexity in linking AMIEs to causality.

There are three types of unobserved direct causes of $Y$. We do not consider indirect causes of $Y$ since they do not influence predicting $Y$ given the direct causes. Please refer Section~\ref{sec_summary} for more discussions. 

\begin{enumerate}[leftmargin=0.5cm]
	\item {\bf Unobserved direct cause with an observed parent}:  $U_i$ in Figure~\ref{pic:intro1} is an unobserved direct cause of $Y$ and its parent $X_i$ is observed. We call the observed variable $X_i$ an \emph{activator of the unobserved direct cause} since the effect of $U_i$ on $Y$ is activated or executed when $X_i$ is intervened.  
	\item {\bf Unobserved direct cause with an observed proxy}: As shown in Figure~\ref{pic:intro2},  $U_i$, the unobserved direct cause of $Y$ has a direct effect $X_i\in \mathbf{X}$ that is observed. In such a case, we call $X_i$ a \emph{proxy of the unobserved direct cause}.  
	\item {\bf A standalone unobserved direct cause}: As shown in Figure~\ref{pic:intro3},  $U_i$ is an unobserved cause of $Y$ that does not have an observed parent or child which is in $\mathbf{X}$.
\end{enumerate}
	

When there are standalone unobserved direct causes, a causal interpretation by AMIE is impossible since a model does not include the complete information of the causal mechanism. Standalone unobserved direct causes significantly influence the predictive power of a model too. Therefore, we do not consider standalone unobserved direct causes.

In the context of model explanation, both activators and proxies of unobserved direct causes of $Y$ are considered suitable substitutes for their respective unobserved direct causes of $Y$, and we have the following theorem to link AMIE with causality in the presence of unobserved direct causes of $Y$ with activators or proxies.
\begin{theorem}[Linking AMIE with direct causes in the presence of unobserved variables]
	\label{the:twotypes}
	Suppose that Assumptions~\ref{assump-Data-graph-faithfulness} and~\ref{assump-Data-model-faithfulness} hold and there are no standalone unobserved direct causes of $Y$. If $\mathbf{X}$ does not include any descendants of $Y$, then $\amie(X_i, Y) \ne 0$ for a variable $X_i\in \mathbf{X}$ that is an observed direct cause, an activator or a proxy of an unobserved direct cause of $Y$.
\end{theorem}    

The above theorem presents a necessary condition to link direct causes (or their activators and proxy proxies). The features with non-zero AMIEs may include others. We define `\emph{inducing paths}' and `\emph{relaxed inducing paths}' as follows to help identify the conditions when the false linkages occur.  

\begin{definition}[Inducing Paths and Relaxed Inducing Paths]
	\label{def:inducing}
	{\fontdimen2\font=0.2em The path between $X$ and $Y$ in a causal DAG is inducing if all other observed variables along the path are colliders which are also ancestors of $Y$. The path between $X$ and $Y$ is a relaxed inducing path if it is an inducing path except that the variables next to $Y$ (i.e., the ones that share an unobserved ancestor with $Y$) are not an ancestor of $Y$.}
\end{definition}
	
	
When an inducing or relaxed inducing path exists between $X$ and $Y$, a spurious association is often formed between $X$ and $Y$~\cite{richardson2002ancestral,zhang2008completeness}. These spurious correlations lead to false identification of direct cause of $Y$, activates or proxies of unobserved direct causes of $Y$ from the data. Therefore, we propose the following theorem addressing the issue of false identifications.

\begin{theorem}[Cases for false linkages of AMIEs with causality]
	\label{theorem_false01}
Given that Assumptions~\ref{assump-Data-graph-faithfulness} and~\ref{assump-Data-model-faithfulness} hold. There are no standalone unobserved direct causes of $Y$ and $\mathbf{X}$ does not include any descendants of $Y$. If $\amie(X_j, Y) \ne 0$ but $X_j$ is not a direct cause of $Y$, or an activator or proxy of an unobserved direct cause of $Y$, one of the following cases will be true: Case 1 $X_j$ is a parent of the proxy of the unobserved direct cause of $Y$; Case 2 $X_j$ shares a common unobserved ancestor with the proxy of the unobserved direct cause of $Y$; or Case 3 $X_j$ and $Y$ form an inducing path or a relaxed inducing path. 
\end{theorem}

\begin{figure}[t]
	\centering
	\begin{subfigure}[c]{0.2\textwidth}
		\centering
		\includegraphics[scale=0.3]{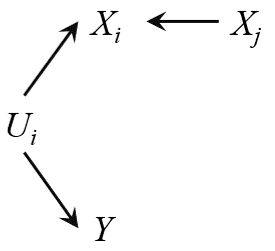}
		\caption{}
		\label{pic:examplefalse1}
	\end{subfigure}
	\begin{subfigure}[c]{0.2\textwidth}
		\centering
		\includegraphics[scale=0.3]{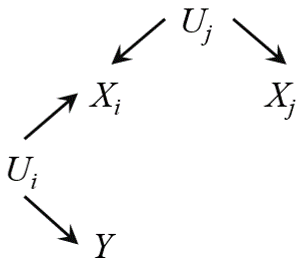}
		\caption{}
		\label{pic:examplefalse2}
	\end{subfigure}
	\begin{subfigure}[c]{0.25\textwidth}
		\centering
		\includegraphics[scale=0.3]{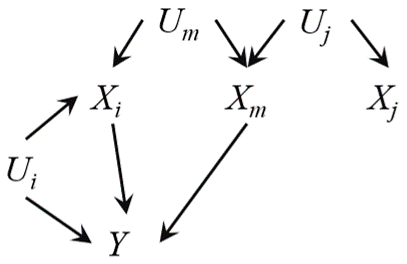}
		\caption{}
		\label{pic:examplefalse3}
	\end{subfigure}
	\begin{subfigure}[c]{0.25\textwidth}
		\centering
		\includegraphics[scale=0.3]{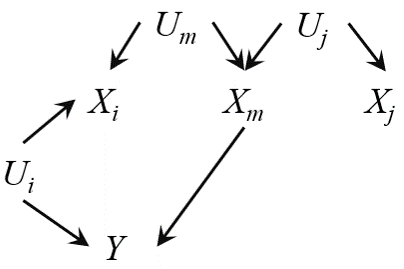}
		\caption{}
		\label{pic:examplefalse4}
	\end{subfigure}
	\caption{Four exemplar causal DAGs illustrate three false positive cases in Theorem~\ref{theorem_false01}. Specifically: (a) Case 1, $X_j$ is a parent of $X_i$, and $X_i$ is a proxy of the unobserved direct cause $U_i$; (b) Case 2, $X_j$ shares a common unobserved confounder $U_j$ with $X_i$ which is a proxy of the unobserved direct cause $U_i$; (c) Case 3, the inducing path between $X_j$ and $Y$, where $X_m$ is a collider and a cause of $Y$; (d) Case 3, the relaxed inducing path between $X_j$ and $Y$ where $X_m$ is a collider and a cause of $Y$.}
	\label{fig:examplefalse}
\end{figure}


Theorem~\ref{theorem_false01} indicates that when there are unobserved direct causes in the feature set  $\mathbf{X}$, if we use the necessary condition stated in Theorem~\ref{theorem_DirectCause-AMIEs01} to link non-zero AMIEs with direct causes, or activators or proxies of observed direct causes of $Y$, false positives may occur.
 We will now show how to exclude some false linkages.  
 
	
\begin{theorem}[A test for false linkages]
	\label{theo:dfidc01}
	Given that Assumptions~\ref{assump-Data-graph-faithfulness} and~\ref{assump-Data-model-faithfulness} hold. There are no standalone unobserved direct causes of $Y$ and $\mathbf{X}$ does not include any descendants of $Y$. False discovered $X_j$ in Cases (1) and (2) of Theorem~\ref{theorem_false01} can be detected since $X_j \CI Y$ holds.
\end{theorem}

Based on Theorem~\ref{theo:dfidc01}, Cases (1) and (2) false positives can be removed by conducting independence tests. If a feature is independent of $Y$, it can be excluded from linking to a direct cause, or activator or proxy of a direct cause of $Y$. Thus, only Case 3 false positives may remain.  

\subsection{Summarising theoretical results and discussions}
\label{sec_summary}
\begin{wrapfigure}{r}{0.6\textwidth}
	\centering
	\begin{subfigure}[b]{0.15\textwidth}
		\centering
		\includegraphics[scale=0.3]{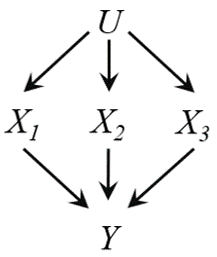}
		\caption{}
		\label{fig:discussions1}
	\end{subfigure}
	\begin{subfigure}[b]{0.15\textwidth}
		\centering
		\includegraphics[scale=0.3]{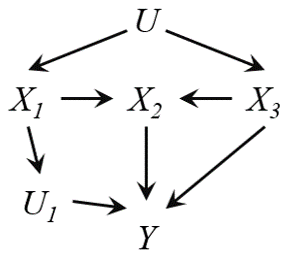}
		\caption{}
		\label{fig:discussions2}
	\end{subfigure}
	\begin{subfigure}[b]{0.15\textwidth}
		\centering
		\includegraphics[scale=0.3]{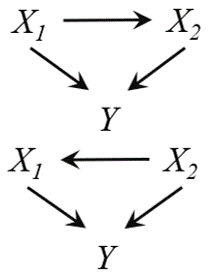}
		\caption{}
		\label{fig:discussions3}
	\end{subfigure}
	\begin{subfigure}[b]{0.1\textwidth}
		\centering
		\includegraphics[scale=0.3]{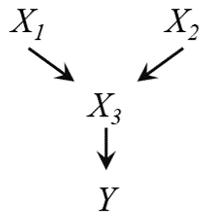}
		\caption{}
		\label{fig:discussions4}
	\end{subfigure}
	\caption{Four exemplar causal DAGs used in the discussions in Section~\ref{sec_summary}.}
	\label{fig:discussions}
\end{wrapfigure}
A summary of the conditions for using AMIE for a causal interpretation of a model is given in Table~\ref{tab-summary}. Whether there are unobserved direct causes of $Y$ or not in $\mathbf{X}$ is a major factor for causal interpretation by AMIE. In the ideal situation where there are no unobserved direct causes of $Y$, AMIE can be used for causal interpretation. When there are unobserved direct causes that have activators or proxies in $\mathbf{X}$, AMIE can still be used for causal interpretation, but with possible false positives. The false positives correspond to Case 3 in Theorem~\ref{theorem_false01}. We will show that false positives in Case 3 are rare in experiments. Therefore, AMIE still provides causal interpretation with a low risk of false positives. When there are unobserved standalone direct causes, AMIE is not applicable anymore. In this case, the model accuracy will be low since some direct causes of $Y$ are missing and their information is not in the model. It is not advisable to explain an inaccurate machine learning model with AMIE. 

\begin{wraptable}{r}{0.73\textwidth}
	\caption{A summary of the conditions for causal interpretation by AMIE.}
	\label{tab-summary}
	\centering
	\setlength\tabcolsep{3pt}
	{\scriptsize \begin{tabular}{|c|c|c|c|}
			\toprule
			\multirow{3}{*}{Assumptions~\ref{assump-Data-graph-faithfulness} and~\ref{assump-Data-model-faithfulness}} & \multirow{3}{*}{\begin{tabular}[c]{@{}c@{}}Without unobserved\\ direct causes\end{tabular}} & \multicolumn{2}{c|}{With unobserved direct causes} \\ \cmidrule{3-4} 
			& & \begin{tabular}[c]{@{}c@{}}without standalone\\ direct causes\end{tabular} & \begin{tabular}[c]{@{}c@{}}with standalone\\ direct causes\end{tabular} \\ \midrule No descendant variables of $Y$ & \textbf{Yes} & \textbf{Yes with false positives} & \textbf{No} \\ \bottomrule
	\end{tabular}}
\end{wraptable}
We now discuss how previous causality-based feature attribution methods fit in the conditions in Table~\ref{tab-summary} and their difference from the results in this paper. 

Methods~\cite{goyal2019explaining,Feder-CausaLM-2021,abraham2022cebab} fall in column two of Table~\ref{tab-summary}, i.e., without unobserved direct causes.

%
The causal graph used (or assumed) by the methods is shown in Figure~\ref{fig:discussions1} where $\mathbf{X} \backslash X_i$ is the confounder set when feature $X_i$ is considered as the treatment or being attributed. This is a special causal DAG and does not represent a general case. To generally discuss causality in explanation, $do$-calculus has been used to define a causal effect, mostly incorporated into the causal Shapley values~\cite{heskes2020causal,janzing2020feature}. To derive the probability expression of a $do$-expression, there needs a known causal graph or causal orders of all variable pairs. Such causal knowledge is not available in many applications, and this limits the application scope of the methods. These works do not consider unobserved direct causes and hence fall in column two of Table~\ref{tab-summary} too.  

Work~\cite{pmlr-v162-jung22a} estimating $do$-Shapley values belongs to the third column in Table~\ref{tab-summary}, i.e., with unobserved direct causes but without standalone direct causes. The causal graph representing the problem where $do$-Shapley values can be estimated is shown in Figure~\ref{fig:discussions2}. The causal orders among $X_1$, $X_2$ and $X_3$ are necessary to estimate $do$-Shapley values. The unobserved variables considered in work~\cite{pmlr-v162-jung22a} do not include the standalone unobserved direct causes. The work does not consider the proxy of an unobserved direct cause either since the effect a proxy variable on $Y$ is unidentifiable using their identification theorem. In our view, proxy variables of unobserved direct causes are important for the explanation.

The previously discussed causal interpretable methods either assume a simple causal graph as Figure~\ref{fig:discussions1} and no unobserved direct causes or need a causal graph or the causal orders of all variables. Either a causal graph or causal ordering knowledge is not available in many applications. It is impossible to learn a unique complete DAG from data~\cite{spirtes2000causation,maathuis2009estimating} either. In contrast, The AMIE-based explanation does not need a DAG or the total causal orders of $X$s (See discussions after Theorem~\ref{theorem_DirectCause-AMIEs01}. 
Using AMIE-based explanation is able to identify direct causes, activators or proxies of unobserved direct causes of $Y$ without a complete DAG but with the conditions summarised in Table~\ref{tab-summary}. 

We now differentiate AMIEs from Shapley values. Shapley values are mainly used for local or prediction level explanation~\cite{pmlr-v119-sundararajan20b}, which is different from model explanation focused in this work. A few works use Shapley values for model explanation, but explanation is not causal~\cite{Gromping-Importance-2007,Owen-OnShapley-2017,Song-Shapley-Global-16,Covert-SAGE-2000}. One major weakness of Shapley value is its time complexity for calculating. For example, computing the average Shapley value of a feature involves $n*2^{m-1}$ conditional probability estimations from the model where $n$ represents the number of instances and $m$ the number of features. In contrast, estimating AMIE of a feature only involves $n*m$ conditional conditional probabilities read from the model. Note that works~\cite{Song-Shapley-Global-16,Covert-SAGE-2000} improved efficiency greatly. 
%
A strength of Shapley values in model explanation is that it considers feature interactions. However, the interaction considered by Shapley values without a causal graph is problematic. Let us consider two causal DAGs in Figure~\ref{fig:discussions3}, $X_1$ and $X_2$ have an interaction. $X_1$ and $X_2$ will share a half of the contribution of interaction using the Shapley values discussed in~\cite{Gromping-Importance-2007,Owen-OnShapley-2017,Song-Shapley-Global-16}. Such an attribution is symmetric in both causal graphs but this should not be. In the upper graph $X_1$ should take the interaction effect while in the bottom graph $X_2$ should take the interaction effect. Interaction in a causal view should be considered with a causal graph. AMIE does consider interactions since it does not assume the causal knowledge among features which is unavailable in many applications.

AMIEs represent the average controlled direct effects and the AMIE of an indirect cause is zero. This makes sense because the set of direct causes completely explains $Y$. Given the set of direct causes, all other non-descendant variables of $Y$ are independent of the $Y$, i.e., they do not contribute to predicting $Y$. Moreover, if an explanation method aims to explain indirect causes, path-specific effects are necessary to achieve a fair feature attribution. For example, in Figure~\ref{fig:discussions4}, when the effects of $X_1$ and $X_2$ on $Y$ are estimated, the effect of $X_3$ should not be double counted. To achieve feature attribution without multiple countering some direct effects, a complete DAG is essential. However, a DAG is unavailable in most applications.

\section{Experiments}	
\label{sec:exp}
In this section, we will first demonstrate the correctness of theorems. Then, we show the usefulness of AMIE using two semi-synthetic datasets. Note that comparison is not the aim of experiments since the main contribution of the paper is to identify conditions linking intervention intuition to causality. Most existing feature attribution methods use the intervention intuition already and they will reveal some direct causes too when the conditions are met. The validation of conditions identified by theorems is the major aim of experiments.  




\subsection{Linking AMIES with direct causes without unobserved variables}
\label{Unobserved}
\begin{wraptable}{r}{0.6\textwidth}
	\setlength\tabcolsep{3pt}
	\caption{The consistency ($\%$) between direct causes of $Y$ and features with non-zero AMIEs without unobserved variables.}
	\label{tab:Completeness1}
	\centering
	{\scriptsize\begin{tabular}{cccccccc}
			\toprule
			\multirow{2}{*}{} & \multicolumn{3}{c}{LR} & &\multicolumn{3}{c}{RF} \\ \cmidrule{2-4} \cmidrule{6-8}
			& $d$ = 2 & $d$ = 4 & $d$ = 6 & & $d$ = 2 & $d$ = 4 & $d$ = 6 \\ \midrule
			$n$ = 40 & 98.0 ± 6.0 & 92.9 ± 8.8 & 92.1 ± 8.9 & & 98.0 ± 6.0 & 97.0 ± 6.4 & 93.5 ± 9.0 \\
			$n$ = 60 & 98.0 ± 6.0 & 95.3 ± 7.5 & 93.0 ± 9.1 & & 99.0 ± 3.0 & 98.0 ± 6.0 & 94.0 ± 8.4 \\
			$n$ = 80 & 99.0 ± 3.0 & 95.9 ± 6.4 & 93.6 ± 7.9 & & 100.0 ± 0.0 & 96.9 ± 6.6 & 95.2 ± 7.7 \\ \bottomrule
	\end{tabular}}
\end{wraptable}

In this section, we design an experiment to demonstrate the correctness of Theorem~\ref{theorem_DirectCause-AMIEs01}. We first randomly generate DAGs by using some parameters, including the number of observed nodes ($m+1$ for short) and the ratio of edges and nodes ($d$ for short). In our experiment, $m+1$ is set to 40, 60, and 80, while $d$ is set to 2, 4, and 6. When generating DAGs, we set the last node as $Y$ and ensure no children nodes of $Y$. We can access the ground truth (i.e., the direct causes of the outcome) since the DAGs are available. We use the above DAGs to generate synthetic datasets with 10,000 instances each. We select two well-known classification methods, Logistic Regression (LR)~\cite{hosmer2013applied} and Random Forest (RF)~\cite{breiman2001random}, to train the prediction models.

For evaluating the consistency between direct causes and features of non-zero AMIEs, we use consistency $|\mathbf{S_1} \cup \mathbf{S_2}|/|\mathbf{S_1}|$ as a measure, where $\mathbf{S_1}$ is the set of true direct causes obtained by DAGs and $\mathbf{S_2}$ is the set of features with non-zero AMIEs. $|\mathbf{S_*}|$ indicates the cardinality of set $\mathbf{S_*}$. To mitigate the bias brought by the data generation process, we repeatedly generated 30 datasets. We report the average consistency over the 30 datasets with the standard deviation.

From Table~\ref{tab:Completeness1}, we observe that the consistencies between the set of direct causes and the set of features with non-zero AMIEs are close to 100\%. This demonstrates the correctness of Theorem~\ref{theorem_DirectCause-AMIEs01}.  

\subsection{{\small Linking AMIES with direct causes with unobserved variable but no standalone direct causes}}	
\label{non-Independent}
In this section, we conduct experiments to evaluate the correctness of Theorems~\ref{the:twotypes},~\ref{theorem_false01}, and~\ref{theo:dfidc01}. The experiment procedure is similarly to the previous section. In addition, the situations where DAGs have unobserved variables except unobsered standalone direct causes are allowed. The number of unobserved variables ($l$ for short) is set by 2, 4, and 6, and they are randomly selected from all nodes except the outcome. For each setting (the number of observed nodes, the ratio of edges and nodes, and the number of unobserved nodes), we randomly generate 100 DAGs for evaluation.  

\begin{wraptable}{r}{0.65\textwidth}
	\centering
	\caption{The consistency ($\%$) between direct causes (including activators and proxies of unobserved direct causes) of $Y$ and features with non-zero AMIEs with unobserved variables but no unobserved standalone direct causes.}
	\label{tab:Completeness3}
	\setlength\tabcolsep{3pt}
	{\scriptsize \begin{tabular}{ccccccccc}
			\toprule
			\multirow{2}{*}{} & & \multicolumn{3}{c}{LR} & &\multicolumn{3}{c}{RF} \\ \cmidrule{3-5} \cmidrule{7-9}
			&$l$ & $d$ = 2 & $d$ = 4 & $d$ = 6 & & $d$ = 2 & $d$ = 4 & $d$ = 6 \\ \midrule
			\multirow{3}{*}{$n$ = 40} & 2  & 97.0 ± 4.6  & 88.8 ± 9.7  & 85.8 ± 10.1 && 99.0 ± 3.0  & 91.0 ± 9.4  & 89.0 ± 9.4  \\
			& 4 & 95.0 ± 6.7  & 86.6 ± 9.4  & 84.8 ± 9.1 && 97.0 ± 4.6  & 89.0 ± 9.4  & 88.0 ± 8.7  \\
			& 6 & 93.5 ± 7.1  & 84.3 ± 8.5  & 83.6 ± 9.1 && 96.5 ± 5.5  & 87.0 ± 9.0  & 87.0 ± 9.0  \\ \midrule
			\multirow{3}{*}{$n$ = 60} & 2  & 95.0 ± 8.1  & 90.2 ± 9.0  & 89.4 ± 8.1  && 97.0 ± 6.4  & 92.3 ± 8.3  & 90.3 ± 9.1\\
			& 4 & 94.8 ± 8.0  & 88.5 ± 8.6  & 87.7 ± 7.4  && 96.0 ± 6.6  & 90.7 ± 8.3  & 88.7 ± 8.6\\
			& 6 & 93.0 ± 7.8  & 86.0 ± 8.5  & 85.4 ± 5.9  && 95.0 ± 6.7  & 88.7 ± 8.2  & 86.3 ± 7.8\\ \midrule
			\multirow{3}{*}{$n$ = 80} & 2  & 97.0 ± 6.4  & 90.0 ± 7.7  & 89.5 ± 6.5 && 99.0 ± 3.2  & 93.5 ± 8.2  & 90.5 ± 9.0 \\
			& 4 & 96.0 ± 6.6  & 88.0 ± 7.5  & 88.0 ± 5.6  && 98.0 ± 4.2  & 91.0 ± 9.7  & 89.0 ± 8.4\\
			& 6 & 94.0 ± 8.0  & 86.0 ± 6.6  & 86.0 ± 4.4  && 95.0 ± 7.1  & 88.5 ± 10.3  & 88.0 ± 7.5\\ \bottomrule
	\end{tabular}}
\end{wraptable}

The consistencies between direct causes (including activators and proxies of unobserved direct causes) and features with non-zero AMIEs are reported in Table~\ref{tab:Completeness3}. With the increase in the number of unobserved variables, the consistency decreases, but still at a high level (above 85\%). Compared with the results in Table~\ref{tab:Completeness1}, the difference indicates the false discoveries due to Case 3 in Theorems~\ref{theorem_false01}. 

The decrease of inconsistencies due to Case 3 in Theorems~\ref{theorem_false01} is small because the occurrence of Case 3 is infrequent. We design an experiment to count the cases of (relaxed) inducing paths in random DAGs with random unobserved variables. The results are reported in Appendix~\ref{APPCompleteness}, and they show that Case 3 is not common in the random DAGs with random unobserved variables.

\subsection{The risk of using AMIEs with unobserved standalone direct causes}
\label{subsec:three}
\begin{wraptable}{r}{0.58\textwidth}
	\caption{Accuracy of the model built from data with unobserved standalone direct causes and consistency between the set of direct causes and features with non-zero AMIEs.}
	\label{tab:Completeness4}
	\centering
	{\scriptsize \begin{tabular}{cccccc}
			\toprule
			&  &  \multicolumn{2}{c}{LR} & \multicolumn{2}{c}{RF} \\ \midrule
			$n$ & $l$    &  Model Acc & Rec Rate &  Model Acc & Rec Rate \\ \midrule
			\multirow{3}{*}{40} & 2  & 78.9 ± 14.7  & 82.9 ± 20.9  & 80.5 ± 16.6  & 85.7 ± 22.0  \\
			& 4 & 70.0 ± 10.8  & 77.7 ± 21.5  & 71.0 ± 13.8  & 79.0 ± 21.2  \\
			& 6 & 65.0 ± 7.4  & 70.4 ± 28.4  & 67.9 ± 10.6  & 73.9 ± 29.6  \\ \midrule
			\multirow{3}{*}{60} & 2  & 76.8 ± 11.2  & 81.2 ± 20.2  & 80.7 ± 11.3  & 84.0 ± 21.5  \\
			& 4 & 71.9 ± 10.6  & 74.3 ± 20.6  & 73.1 ± 13.4  & 75.0 ± 20.6  \\
			& 6 & 64.1 ± 7.0  & 68.8 ± 29.1   & 68.4 ± 10.0  & 71.4 ± 28.3  \\ \midrule
			\multirow{3}{*}{80} & 2  & 78.0 ± 8.1  & 80.0 ± 20.4  & 81.4 ± 8.9  & 85.7 ± 12.1  \\
			& 4 & 70.9 ± 9.2  & 75.7 ± 20.2  & 72.4 ± 14.2  & 76.0 ± 19.6  \\
			& 6 & 63.3 ± 6.9  & 70.0 ± 28.4   & 66.2 ± 11.1  & 71.4 ± 30.5  \\ \bottomrule
	\end{tabular}}
\end{wraptable}

In this section, we demonstrate the risk of using AMIEs when the main condition without unobserved standalone direct causes is violated. We fix the ratio of edges and nodes, i.e., $d$ = 4. Different from previous settings, the unobserved variables are standalone, i.e., $U \to Y$.  From Table~\ref{tab:Completeness4}, we observe that consistency between the direct causes and the features with non-zero AMIEs decreases with the increase of the number of the standalone unobserved variables, and the model accuracy also has the same trend. Compared with the results shown in Section~\ref{Unobserved} and Section~\ref{non-Independent}, as the number of unobserved standalone direct causes increases, causal interpretation with AMIE is highly uncertain.

\subsection{A demonstration of usefulness of AMIEs in semi-synthetic datasets}	
\label{Evaluation on Semi-Synthetic}
In this section, we aim to show the usefulness of AMIE. The reason to use semi-synthetic datasets is that such datasets based on Bayesian networks have ground truth direct causes. We first split training and test data with 70\%/30\% proportions, and select two well-known classification methods, Logistic Regression (LR) and Random Forest (RF), to train the prediction models. Secondly, we apply the above prediction models to test datasets and calculate AMIEs. Finally, we rank the features by using the value of AMIEs to test the consistency with the ground truth direct causes in comparison with feature importance provided by the models themselves, logistic regression coefficients and permutation feature importance provided by RF function~\cite{breiman2001random}. We also list results by Shapley effect~\cite{Song-Shapley-Global-16} given Shapley values are so popular in XAI.

\paragraph{Insurance.} The Bayesian Network Insurance~\cite{bnlearn} contains 26 nodes and 50 edges. All the variables are discrete, we select ThisCarCost as the outcome, and the direct causes of the outcome are ThisCarDam, CarValue, and Theft. We apply the one-hot encoder on all features, and then we get the dataset with 80 features. The Bayesian network for Insurance is shown in Appendix~\ref{APPinsurance}.

\begin{table*}[t]
\setlength\tabcolsep{2pt}
    \caption{Experiment results on Insurance. The upper table shows the results of LR, while the lower table shows the results of RF. The top ten important features by AMIE, model importance ranking and Shapley effect are listed. The true direct causes are checked.}
    \label{tab:Insurance}
    \centering
    {\scriptsize \begin{tabular}{cccccccccc}
    \toprule
    \multicolumn{10}{c}{LR (model accuracy 98.6\%)}                                                                                      \\ \midrule
                         & Ranked by AIME       & AMIE        &           & Ranked by Imp            & Coefficient  & & Ranked by Shapley Effect            & Shapley Effect  &      \\ \midrule
    1                    & ThisCarDam\_Severe   & 0.1749 &\checkmark       & ThisCarDam\_Severe       & 3.9611    &\checkmark & ThisCarDam\_Severe
      & 0.1485    &\checkmark                      \\
    2                    & ThisCarDam\_None     & 0.1484 &\checkmark       & Theft\_True              & 3.0283    &\checkmark & ThisCarDam\_Mild
       & 0.0792    &\checkmark                       \\
    3                    & ThisCarDam\_Mild     & 0.0917 &\checkmark       & ThisCarDam\_Moderate     & 2.3946    &\checkmark & CarValue\_FiveThou
       & 0.0437    &\checkmark                      \\
    4                    & Theft\_False         & 0.0652 &\checkmark       & Accident\_Moderate       &  0.9411   & & DrivQuality\_Excellent
       & 0.0272    &                      \\
    5                    & Theft\_True          & 0.0652 &\checkmark       & Accident\_Severe         & 0.8305    & & ThisCarDam\_Moderate
       & 0.0227    &\checkmark                      \\
    6                    & ThisCarDam\_Moderate & 0.0650 &\checkmark       & RiskAversion\_Psychopath &  0.6208   & & HomeBase\_Suburb
       & 0.0215    &              \\
    7                    & Accident\_None       & 0.0633 &                 & ILiCost\_HundredThou     &  0.6057   & & Accident\_Severe
       & 0.0199    &                    \\
    8                    & Accident\_Moderate   & 0.0176 &                 & Accident\_Mild           &  0.5107   & & DrivQuality\_Poor
       & 0.0161    &                    \\
    9                    & Accident\_Severe     & 0.0152 &                 & MedCost\_TenThou         &  0.4463   & & Mileage\_Domino
       & 0.0140    &                   \\
    10                   & ILiCost\_Million     & 0.0115 &                 & CarValue\_FiftyThou      &  0.4201   &\checkmark & Accident\_Mild
       & 0.0138    &     \\ \bottomrule
    \multicolumn{1}{l}{} &\multicolumn{1}{l}{}  & \multicolumn{1}{l}{} & \multicolumn{1}{l}{} & \multicolumn{1}{l}{}     & \multicolumn{1}{l}{}& \multicolumn{1}{l}{}& \multicolumn{1}{l}{}& \multicolumn{1}{l}{}& \multicolumn{1}{l}{} \\ \toprule
    \multicolumn{10}{c}{RF (Model Accuracy  98.6\%)}                                                                                      \\ \midrule
                         & Ranked by AIME        & AMIE    &       & Ranked by Imp         & Coefficient  & & Ranked by Shapley Effect            & Shapley Effect  &      \\ \midrule
    1                    & ThisCarDam\_Severe    & 0.2612  &\checkmark        & ThisCarDam\_Severe     &0.1828  &\checkmark & Accident\_None
     &0.0088  &                      \\
    2                    & ThisCarDam\_None      & 0.1966  &\checkmark        & ThisCarDam\_None       &0.1169       &\checkmark & Accident\_Mild
     &0.0063  &                      \\
    3                    & ThisCarDam\_Moderate  & 0.1875  &\checkmark        & Accident\_Severe     &0.0765 & & ThisCarDam\_Mild
     &0.0060  &\checkmark                      \\
    4                    & Accident\_Severe      & 0.1728  &                  & Accident\_None       & 0.0720  & & ThisCarDam\_Severe
     &0.0057  &\checkmark                     \\
    5                    & Accident\_Moderate    & 0.1584  &                  & OtherCarCost\_Thousand  & 0.0719        & & ThisCarDam\_None
     &0.0042  &\checkmark                     \\
    6                    & OtherCarCost\_Thousand& 0.1394  &                  & ThisCarDam\_Moderate   &0.0714 &\checkmark & ThisCarDam\_Moderate
     &0.0038  &\checkmark                     \\
    7                    & Theft\_True           & 0.1319  &\checkmark        & ThisCarDam\_Mild    &0.0573   &\checkmark & OtherCarCost\_HundredThou
     &0.0024  &                     \\
    8                    & MedCost\_Thousand     & 0.1263  &                  & Accident\_Moderate   & 0.0468      &  & OtherCarCost\_TenThou
     &0.0019  &                    \\
    9                    & Accident\_None        & 0.1118  &                  & OtherCarCost\_TenThou   & 0.0360      & & Accident\_Severe
     &0.0015  &                     \\
    10                   & Theft\_False          & 0.1096  &\checkmark        & MedCost\_Thousand      &   0.0277 & & OtherCarCost\_Thousand
     &0.0015  &                  \\ \bottomrule
    \end{tabular}}
\end{table*}

Table~\ref{tab:Insurance} reports the experimental results on the insurance dataset. We have the following observations.

\begin{enumerate}[leftmargin=0.5cm]
    \item The models' views on the same dataset are quite different. We see that the ranks of AMIE for LR and RF are different. Based on the AMIE ranking, the model of LR uses direct causes better than RF. This might mean that the LR model is more faithful to the underlying data-generating mechanism than RF on this dataset. Note that both models have the same accuracy. In real applications, if domain experts know the direct causes, AMIE will help domain experts pick a faithful model among several alternative models. 

    \item By comparing the number and ranks of true direct causes included in the top ten features ranked by AMIE, logistic regression coefficient/permutation feature importance~\cite{breiman2001random} and Shapley effect~\cite{Song-Shapley-Global-16}, the ranking by AMIE is closer to the ground truth. 
    \end{enumerate}

\paragraph{Water.} The conclusions drawn from the Water dataset are very similar to those from the Insurance dataset. However, fewer direct causes were included in the Water dataset than in the Insurance dataset. This might be due to the lower model accuracy on Water, which affects the linkage between AMIE and direct causes because of a lack of faithfulness between model and underlying causal mechanism. Detailed results of Water dataset are presented in Appendix~\ref{APPwater}. 


\section{Conclusion}
\label{sec:con}
In this work, we have studied a linkage between an intuitive intervention in model explanation, AMIE, and its causal interpretation, and identified conditions for such a linkage. When conditions are satisfied, features with non-zero AMIEs are direct causes, activators and proxies of direct causes and such a feature attribution enables users to choose a trustworthy model among a few alternative models for reliable deployment. The paper also studied the risks for such causal interpretation when unobserved direct causes are presented. Extensive experiments have demonstrated the soundness of the theorems and the potential of AMIE in real-world applications.

\bibliographystyle{plain}
\bibliography{amie.bib}

	
\appendix
\section{Background}\label{Appendix:A}
The following conditions/assumptions are commonly used in causal graphical modelling.	
\begin{assumption} [Markov Condition~\cite{spirtes2000causation}]
    \label{asm_Markovcondition}
    Given a DAG $\mathcal{G}=(\mathbf{V}, \mathbf{E})$ and $P(\mathbf{V})$, the joint probability distribution of $\mathbf{V}$, $\mathcal{G}$ satisfies the Markov condition if for $\forall X_i \in \mathbf{V}$, $X_i$ is independent of all non-descendants of $X_i$, given the set of parent nodes of $X_i$.
\end{assumption} 
	
With a DAG satisfying the Markov condition, the joint distribution of $\mathbf{V}$ can be factorised as $P (\mathbf{V}) = \prod_i P(X_i | Pa (X_i))$, where $Pa(X_i)$ denotes the set of direct causes of $X_i$.
	
With the Assumption~\ref{assump-Data-graph-faithfulness}, a causal DAG and data are linked. We can read the (in)dependencies between variables in $P(\mathbf{V})$ from a DAG $\mathcal{G}$ using $d$-separation~\cite{pearl2009causality}. 

In a DAG, d-separation is a well-known graphical criterion that is used to read off the identification of conditional independence between variables entailed in the DAG when the Markov property, faithfulness and causal sufficiency are satisfied~\cite{pearl2009causality,spirtes2000causation}.

\begin{definition}[d-separation~\cite{pearl2009causality}]
	\label{d-separation}
	A path $\pi$ in a DAG $\mathcal{G}=(\mathbf{V}, \mathbf{E})$ is said to be d-separated (or blocked) by a set of nodes $\mathbf{Z}$ if and only if
	(i) $\pi$ contains a chain $X_i \rightarrow X_k \rightarrow X_j$ or a fork $X_i \leftarrow X_k \rightarrow X_j$ such that the middle node $X_k$ is in $\mathbf{Z}$, or
	(ii) $\pi$ contains a collider $X_k$ such that $X_k$ is not in $\mathbf{Z}$ and no descendant of $X_k$ is in $\mathbf{Z}$.
	A set $\mathbf{Z}$ is said to d-separate $X_i$ from $X_j$ ($X_i \CI_{d} X_j|\mathbf{Z}$) if and only if $\mathbf{Z}$ blocks every path between $X_i$ to $X_j$. Otherwise they are said to be d-connected by $\mathbf{Z}$, denoted as $X_i\nCI_{d} X_j|\mathbf{Z}$.
\end{definition}

\begin{property}
	Two observed variables $X_i$ and $X_j$ are d-separated given a conditioning set $\mathbf{Z}$ in a DAG if and only if $X_i$ and $X_j$ are conditionally independent given $\mathbf{Z}$ in data~\cite{spirtes2000causation}. If $X_i$ and $X_j$ are d-connected, $X_i$ and $X_j$  are conditionally dependent. 
\end{property}
	
\begin{definition}[Causal Sufficiency~\cite{spirtes2000causation}]
    \label{def:causuf}
    A given dataset satisfies causal sufficiency if, for every pair of observed variables, all their common causes are observed.
\end{definition}
	
However, causal sufficiency is often violated in practice, as it is not feasible to operate in a closed world where all related variables are collected~\cite{spirtes2000causation,zhang2008completeness}. In this work, we employ a causal DAG denoted as $\mathcal{G}=(\mathbf{V}, \mathbf{E})$ with $\mathbf{V}=\{\mathbf{X}, \mathbf{U}\}$  to represent the causal relationships between variables, where $\mathbf{X}$ is the set of measured variables and $\mathbf{U}$ represents the set of unobserved variables.


\section{Proofs}
\label{sec:MIE_app}

\setcounter{theorem}{0}
\begin{theorem}[Linking AMIEs and direct causes of $Y$ without the presence of unobserved variables]
\label{theorem_DirectCause-AMIEs01_app}
	Suppose that Assumptions~\ref{assump-Data-graph-faithfulness} and~\ref{assump-Data-model-faithfulness} hold, and there are no unobserved variables. If $\mathbf{X}$ includes no descendant variables of $Y$, then $\forall X_i\in\mathbf{X}$, $\amie(X_i, Y) \ne 0$ if and only if $X_i$ is a direct cause of $Y$. Equivalently, $\amie(X_i, Y) = 0$ if and only if $X_i$ is not a direct cause of $Y$.
\end{theorem}


\begin{proof}
	Given the assumptions, all and only the conditional independencies in the data are encoded in the DAG $\mathcal{G}$ and can be read from $\mathcal{G}$.  
	
	Firstly, we prove that if $\amie(X_i, Y) \ne 0$, then $X_i$ is a direct cause of $Y$. Based on Definition~\ref{def:amie}, $\amie(X_i, Y) \ne 0$ implies that there exists a causal effect of $X_i$ on $Y$ conditioning on all other variables. In other words, there must be an edge between $X_i$ and $Y$, otherwise conditioning on all other variables $\mathbf{X}'$, $X_i$ and $Y$ are independent, i.e., AMIE must be zero as when $X_i$ and $Y$ and independent conditioning on $\mathbf{X}'$, $P(Y| X_i=1, \mathbf{X}')=P(Y| X_i=0, \mathbf{X}')$, where $\mathbf{X}'\cup X =\mathbf{X}$. Thus, $X_i$ must be a direct cause of $Y$ since $Y$ has no descendants.
	
   Next, we show that if $X_i$ is a direct cause of $Y$, then $\amie(X_i, Y) \ne 0$. If $X_i$ is a direct cause of $Y$, then given all other variables, $X_i$ and $Y$ are still dependent, i.e., $P(Y| X_i=1, \mathbf{X}') \ne P(Y| X_i=0, \mathbf{X}')$, i.e., $\amie(X_i, Y) \ne 0$.

	Therefore, we have that $\amie(X_i, Y)$ is non-zero if and only if $X_i$ is a direct cause of $Y$. Equivalently, $X_i$ is not a direct cause if and only if $\amie(X_i, Y)$ is zero.
\end{proof}

\setcounter{theorem}{1}
\begin{theorem}[Linking AMIE with direct causes in the presence of unobserved variables]
		\label{the:twotypes_app}
		Suppose that Assumptions~\ref{assump-Data-graph-faithfulness} and~\ref{assump-Data-model-faithfulness} hold and there are no standalone unobserved direct causes of $Y$. If $\mathbf{X}$ does not include any descendants of $Y$, then $\amie(X_i, Y) \ne 0$ for a variable $X_i\in \mathbf{X}$ that is an observed direct cause, an activator or a proxy of an unobserved direct cause of $Y$.
\end{theorem} 


\begin{proof}
 If $X_i$ is an activator of an unobserved direct cause of $Y$, then there exists a direct path $X_i\rightarrow U_i \rightarrow Y$ where $U_i$ is an unobserved direct cause of $Y$.  $X_i$ and $Y$ are not conditional independent given any other observed variables, i.e., $\mathbf{X}'$. This means $P(Y| X_i=1, \mathbf{X}')$ is not equal to $P(Y|X_i=0, \mathbf{X}')$, otherwise, $Y$ and $X_i$ are conditional independent. Therefore $\amie(X_i, Y) \ne 0$. 
 
 If $X_i$ is a proxy of the unobserved direct cause of $Y$, then there exists an unobserved confounder $U_i$ between $X_i$ and $Y$, i.e., $X_i\leftarrow U_i\rightarrow Y$. Based on the structure, $X_i$ and $Y$ are not independent given any other observed variables, then we have $P(Y| X_i=1, \mathbf{X}')$ is not equal to $P(Y|X_i=0, \mathbf{X}')$. Hence, we have $\amie(X_i, Y) \ne 0$.
\end{proof}

\setcounter{theorem}{2}
\begin{theorem}[Potential False Discoveries when Using AMIEs]
	\label{theorem_false01_app}
	Given that Assumptions~\ref{assump-Data-graph-faithfulness} and~\ref{assump-Data-model-faithfulness} hold. There are no standalone unobserved direct causes of $Y$ and $\mathbf{X}$ does not include any descendants of $Y$. If $\amie(X_j, Y) \ne 0$ but $X_j$ is not a direct cause of $Y$, or an activator or proxy of an unobserved direct cause of $Y$, one of the following cases will be true: Case 1 $X_j$ is a parent of the proxy of the unobserved direct cause of $Y$; Case 2 $X_j$ shares a common unobserved ancestor with the proxy of the unobserved direct cause of $Y$; or Case 3 $X_j$ and $Y$ form an inducing path or a relaxed inducing path. 
\end{theorem}

\begin{proof}  
	If $\amie(X_j, Y) \ne 0$, then $P(Y| X_j=1, \mathbf{X}')$ is not equal to $P(Y|X_j=0, \mathbf{X}')$, i.e., $X_j$ and $Y$ are not independent given $\mathbf{X}'$. Thus, we prove that $X_j$ and $Y$ are not independent given $\mathbf{X}'$ in either of the following three cases:
	
	(1)  If $X_j$ is a parent of the proxy $X_i$ of the unobserved direct cause $U_i$ of $Y$, then the path $\pi$, $X_j\rightarrow X_i\leftarrow U_i \rightarrow Y$ in $\mathcal{G}$, i.e., $X_j \nCI_{d} Y| X_i$, and  $\amie(X_j, Y) = \mathbb{E}_{\mathbf{x}} (\mie(X_j, Y |  X_i, \mathbf{X''=x''}))$ where $\mathbf{X}''\setminus \{X_i, X_j\}$. We have $X_j$ and $Y$ are not independent given $X_i\in \mathbf{X}'$ since $X_i$ is a collider on $\pi$.
	
	(2) If $X_j$ shares a common unobserved ancestor $U_j$ with the proxy $X_i$ of the unobserved direct cause $U_i$ of $Y$, then the path $\pi$, $X_j\leftarrow U_j\rightarrow X_i\leftarrow U_i \rightarrow Y$ in $\mathcal{G}$, i.e., $X_j \nCI_{d} Y| X_i$, and  $\amie(X_j, Y) = \mathbb{E}_{\mathbf{x}} (\mie(X_j, Y |  X_i, \mathbf{X''=x''}))$ where $\mathbf{X}''\setminus \{X_i, X_j\}$. Similar to (1), $X_j$ and $Y$ are not independent given $X_i\in \mathbf{X}'$.
	
	(3)  If $X_j$ and $Y$ form an inducing path or a relaxed induing path $\pi$, then $X_j$ and $Y$ are dependent conditioning on a set of measured variables $\mathbf{Q}\subset \mathbf{X}'$, where $\mathbf{Q}$ are colliders and the causes of $Y$ based on Definition~\ref{def:inducing}. Thus, we have $X_j$ and $Y$ are not independent given $\mathbf{Q}$, i.e., $\amie(X_j, Y) = \mathbb{E}_{\mathbf{x}} (\mie(X_j, Y |  \mathbf{Q}, \mathbf{X''=x''}))\ne 0$ where $\mathbf{X}''\setminus \{\mathbf{Q}, X_j\}$. 
\end{proof}



\setcounter{theorem}{3}
\begin{theorem}[A test for false linkages]
	\label{theo:dfidc01_app}
	Given that Assumptions~\ref{assump-Data-graph-faithfulness} and~\ref{assump-Data-model-faithfulness} hold. There are no standalone unobserved direct causes of $Y$ and $\mathbf{X}$ does not include any descendants of $Y$. False discovered $X_j$ in Cases (1) and (2) of Theorem~\ref{theorem_false01} can be detected since $X_j \CI Y$ holds.
\end{theorem}

\begin{proof}
In Cases 1 and 2 of Theorem~\ref{theorem_false01}, on the path between $X_j$ and $Y$, $X_i$ acts as a collider and serves as a proxy of unobserved direct cause $U_i$ of $Y$. Consequently, $X_j\CI Y$ when $X_i$ is not conditioned upon. 
\end{proof}

\section{Experiments}
\label{sec:AppExperiments_app}

\subsection{Completeness of direct causes by AMIE with non-Independent Unobserved Variables}
\label{APPCompleteness}
We design an experiment to count the cases of (relaxed) reducing paths in 100 random DAGs with random non-independent unobserved variables. The frequency of the cases is low. The results are shown in Table~\ref{tab:Completeness444}

\begin{table}[h]
\caption{The number of cases of (relaxed) inducing paths in 100 random DAGs with random non-independent unobserved variables.}
\label{tab:Completeness444}
\centering
{\small \begin{tabular}{ccccccccccccccc}
    \toprule
    $n$                   & $l$ & $d$ = 2 & $d$ = 4 & $d$ = 6 &  & $n$                   & $d$ = 2 & $d$ = 4 & $d$ = 6 &  & $n$                   & $d$ = 2 & $d$ = 4 & $d$ = 6 \\ \cmidrule{1-5} \cmidrule{7-10} \cmidrule{12-15} 
    \multirow{3}{*}{40} & 2 & 6     & 4     & 7     &  & \multirow{3}{*}{60} & 2     & 6     & 3     &  & \multirow{3}{*}{80} & 2     & 2     & 0     \\
                    & 4 & 9     & 10    & 9     &  &                     & 5     & 8     & 5     &  &                     & 2     & 3     & 3     \\
                    & 6 & 16    & 14    & 16    &  &                     & 11    & 9     & 10    &  &                     & 8     & 4     & 8     \\ \bottomrule
\end{tabular}}
\end{table}

\subsection{Evaluation on Insurance Dataset}
\label{APPinsurance}
As shown in Figure~\ref{fig:insurance}, The Bayesian Network Insurance~\cite{bnlearn} contains 26 nodes and 50 edges. All the variables are discrete, we select ThisCarCost as the outcome, and the direct causes of the outcome are ThisCarDam, CarValue, and Theft. We apply the one-hot encoder on all features, and then we get the dataset with 80 features.

\begin{figure}[h]
	\centering
	\includegraphics[scale=0.5]{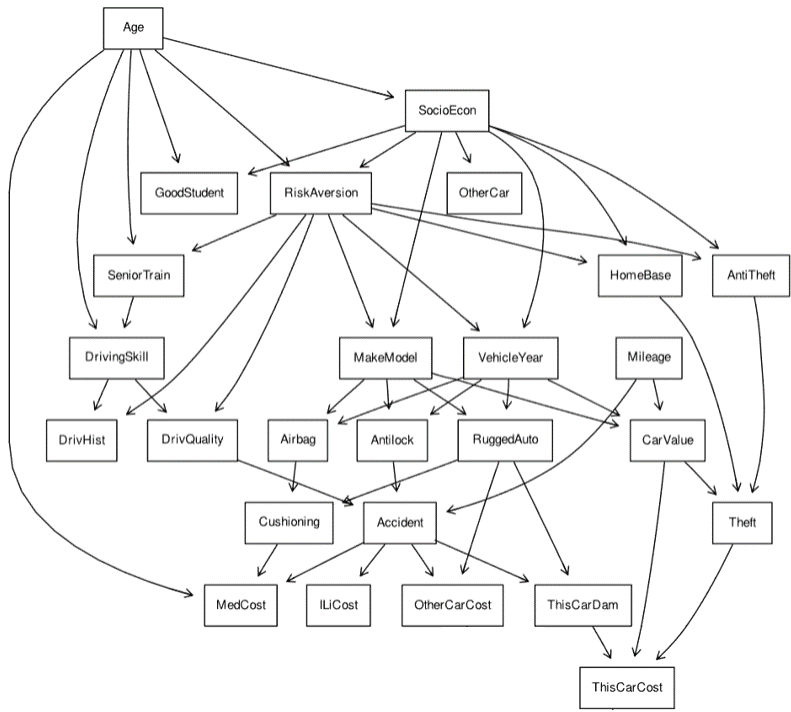}
	\caption{The Bayesian network for Insurance.}
	\label{fig:insurance}
\end{figure}

\begin{figure}[h]
	\centering
	\includegraphics[scale=0.45]{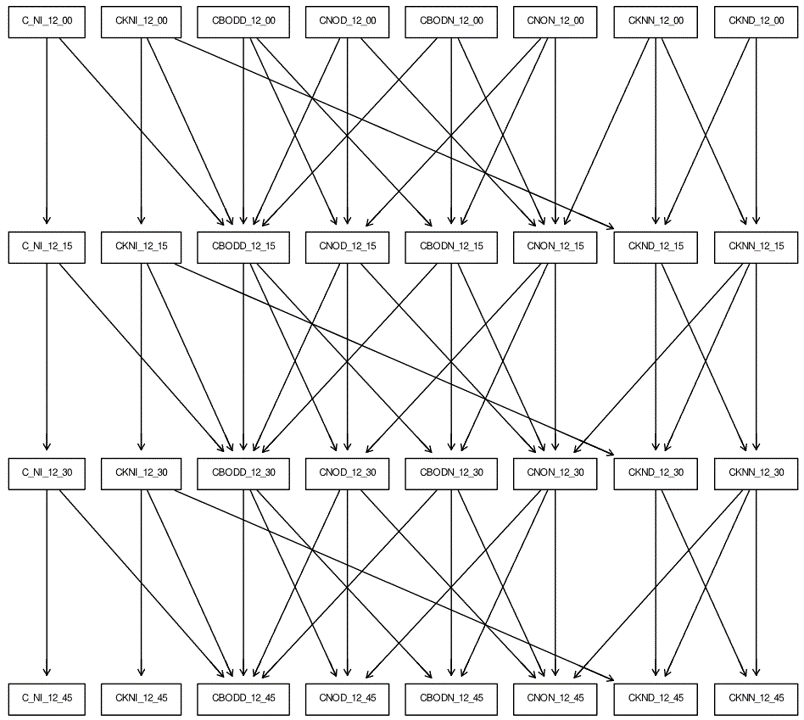}
	\caption{The Bayesian network for Water.}
	\label{fig:water}
\end{figure}

\subsection{Evaluation on Water Dataset}
\label{APPwater}

\paragraph{Water.} As shown in Figure~\ref{fig:water}, the Bayesian Network Water~\cite{bnlearn} contains 32 nodes and 66 edges. All the variables are discrete, we select CNOD\_12\_45 as the outcome, and the direct causes of the outcome are CBODD\_12\_30, CNOD\_12\_30, and CNON\_12\_30. We apply the one-hot encoder on all features, and then we get the dataset with 61 features. Table~\ref{tab:AppWater} reports the experimental results on the insurance dataset.

\begin{table}[h]
\setlength\tabcolsep{2pt}
\caption{Experiment results on Water. The upper table shows the results of LR, while the lower table shows the results of RF. The top ten important features by AMIE, model importance ranking and Shapley effect are listed. The true direct causes are checked.}
\label{tab:AppWater}
\centering
{\scriptsize \begin{tabular}{cccccccccc}
    \toprule
    \multicolumn{10}{c}{LR (model accuracy 82.3\%)}                                                                                      \\ \midrule
                         & Ranked by AIME       & AMIE        &           & Ranked by Imp            & Coefficient  &  & Ranked by Shapley Effect            & Shapley Effect  &     \\ \midrule
    1  & CNOD\_12\_30\_0\_5\_MG\_L & 0.2764 &\checkmark  & CNOD\_12\_30\_1\_MG\_L   & 3.0693 &\checkmark& CNOD\_12\_30\_0\_5\_MG\_L   & 0.2619 &\checkmark  \\
    2  & CNOD\_12\_30\_1\_MG\_L    & 0.2757 &\checkmark  & CBODD\_12\_15\_15\_MG\_L & 1.0707 & & CNOD\_12\_30\_1\_MG\_L   & 0.2609 &\checkmark  \\
    3  & CBODD\_12\_15\_15\_MG\_L  & 0.0876 &  & CBODN\_12\_45\_15\_MG\_L & 0.5025 & & CNOD\_12\_15\_0\_5\_MG\_L   & 0.0568 &  \\
    4  & CBODD\_12\_15\_20\_MG\_L  & 0.0834 &  & CNOD\_12\_15\_1\_MG\_L   & 0.4552 & & CNOD\_12\_15\_1\_MG\_L   & 0.0554 & \\
    5  & CBODD\_12\_30\_25\_MG\_L  & 0.0765 &\checkmark  & CBODN\_12\_30\_10\_MG\_L & 0.3680 & & CKNI\_12\_30\_40\_MG\_L   & 0.0277 & \\
    6  & CNOD\_12\_15\_0\_5\_MG\_L & 0.0481 &  & CBODD\_12\_30\_30\_MG\_L & 0.3359 &\checkmark & CBODN\_12\_45\_15\_MG\_L   & 0.0209 & \\
    7  & CNOD\_12\_15\_1\_MG\_L    & 0.0469 &  & CNON\_12\_15\_6\_MG\_L   & 0.3098 & & CBODD\_12\_30\_25\_MG\_L   & 0.0135 &\checkmark \\
    8  & CBODN\_12\_45\_15\_MG\_L  & 0.0456 &  & CKNI\_12\_30\_40\_MG\_L  & 0.2967 & & CKNI\_12\_30\_30\_MG\_L   & 0.0113 & \\
    9  & CBODN\_12\_45\_5\_MG\_L   & 0.0403 &  & CBODD\_12\_30\_20\_MG\_L & 0.2597 & & CKNI\_12\_45\_40\_MG\_L   & 0.0087 & \\
    10 & CBODN\_12\_30\_15\_MG\_L  & 0.0388 &  & CKND\_12\_15\_6\_MG\_L   & 0.2137 & & CBODD\_12\_15\_20\_MG\_L   & 0.0083 &     \\ \bottomrule
    \multicolumn{1}{l}{} &\multicolumn{1}{l}{}  & \multicolumn{1}{l}{} & \multicolumn{1}{l}{} & \multicolumn{1}{l}{}     & \multicolumn{1}{l}{} &\multicolumn{1}{l}{}&\multicolumn{1}{l}{}&\multicolumn{1}{l}{}&\multicolumn{1}{l}{}\\ \toprule
    \multicolumn{10}{c}{RF (Model Accuracy  85.0\%)}                                                                                      \\ \midrule
                         & Ranked by AIME        & AMIE    &       & Ranked by Imp         & Coefficient  &  & Ranked by Shapley Effect            & Shapley Effect  &     \\ \midrule
    1  & CNOD\_12\_30\_0\_5\_MG\_L & 0.3006 &\checkmark  & CNOD\_12\_30\_0\_5\_MG\_L & 0.2920 &\checkmark & CNOD\_12\_15\_1\_MG\_L   & 0.0124 & \\
    2  & CNOD\_12\_30\_1\_MG\_L    & 0.2285 &\checkmark  & CNOD\_12\_30\_1\_MG\_L    & 0.2492 &\checkmark & CNOD\_12\_15\_0\_5\_MG\_L   & 0.0080 & \\
    3  & CNOD\_12\_15\_1\_MG\_L    & 0.0743 &  & CNOD\_12\_15\_1\_MG\_L    & 0.0513 & & CNOD\_12\_30\_0\_5\_MG\_L   & 0.0045 &\checkmark \\
    4  & CBODD\_12\_15\_15\_MG\_L  & 0.0652 &  & CNOD\_12\_15\_0\_5\_MG\_L & 0.0489 & & CNOD\_12\_30\_1\_MG\_L   & 0.0014 &\checkmark \\
    5  & CNOD\_12\_15\_0\_5\_MG\_L & 0.0632 &  & CKNI\_12\_45\_30\_MG\_L   & 0.0193 & & CKNN\_12\_30\_1\_MG\_L   & 0.0008 & \\
    6  & CKNI\_12\_15\_30\_MG\_L   & 0.0526 &  & CKNI\_12\_00\_30\_MG\_L   & 0.0182 & & CKNI\_12\_00\_40\_MG\_L   & 0.0007 & \\
    7  & CKNI\_12\_45\_40\_MG\_L   & 0.0503 &  & CKNI\_12\_00\_40\_MG\_L   & 0.0181 & & CKNI\_12\_15\_40\_MG\_L   & 0.0007 & \\
    8  & CKNI\_12\_45\_20\_MG\_L   & 0.0500 &  & CKNI\_12\_30\_30\_MG\_L   & 0.0179 & & CKNN\_12\_30\_0\_5\_MG\_L   & 0.0007 & \\
    9  & CKNI\_12\_00\_30\_MG\_L   & 0.0485 &  & CKNI\_12\_15\_30\_MG\_L   & 0.0161 & & CKNI\_12\_30\_20\_MG\_L   & 0.0006 & \\
    10 & CKNI\_12\_00\_40\_MG\_L   & 0.0479 &  & CKNI\_12\_45\_20\_MG\_L   & 0.0154 & & CBODD\_12\_15\_15\_MG\_L   & 0.0005 &                   \\ \bottomrule
    \end{tabular}}
\end{table}

\section{Related Work}
\label{sec:rel}
In this section, we provide a review of existing works on techniques that provide \emph{post-hoc} insights into black-box classification models and the works related to causality-based model explanation. 

\subsection{Non-Causal based Model Explanation}
Most non-causal-based model explanation methods are gradient-based~\cite{selvaraju2017grad,shrikumar2017learning}. These methods can be categorised into two types: local interpretation and global interpretation. Local interpretation techniques, such as LIME (Local Interpretable Model-agnostic Explanations)~\cite{ribeiro2016should} and SHAP (Shapley Additive exPlanations) values~\cite{lundberg2017unified}, focus on creating models that approximate the behaviour of complex models locally. While they provide valuable insights, they also introduce complexities like the need for model selection and the risk of model misspecification. The choice between simpler or more complex explanation models is non-trivial and significantly affects interpretations. These methods require a dataset for generating interpretations, where the availability of original training data or a representative test dataset can be a constraint~\cite{arrieta2020explainable}.

In contrast, global interpretation aims to elucidate a model's decision-making mechanism across the entire input space, aiming to explain the entire decision-making process~\cite{tan2018learning,ibrahim2019global,moraffah2020causal}. For example, some methods utilise the surrogate models for providing global explanations~\cite{lakkaraju2016interpretable,yang2018global}.  Puri et al.~\cite{puri2017magix} employ strategies like generating if-then rules to explain the mechanism of the model. Guo et al.~\cite{guo2018explaining} apply a sophisticated mixture model framework for approximating the underlying model, enabling the extraction of global insights critical for interpreting the target model's behaviour. Tan et al.~\cite{tan2018learning} utilise the technique of model distillation to acquire global additive explanations, which elucidate the connection between input features and the model's predictions. Feature importance estimation of a model is another direction in global interpretation. Permutation feature importance in random forest~\cite{breiman2001random} is a classic method. Shapley value-based feature importance estimator have been studied. Work \cite{Gromping-Importance-2007} discuss the linkage between Shapley values and the feature importance estimation in Linear regression via variance decomposition. Work~\cite{Owen-OnShapley-2017} shows a merit of Shapley value in feature importance estimation, i.e. maintaining the functional equivalence of features. Works~\cite{Song-Shapley-Global-16,Covert-SAGE-2000} improved efficiency for Shapley value-based feature importance indexes at model explanation.


\subsection{Causality-based Model Explanation}
The causal inference has been employed in model explanation to discover which feature (or concept) makes the most important contribution to a model's prediction~\cite{kim2018interpretability,yao2022concept,goyal2019explaining,chattopadhyay2019neural}. 
For example, Yao et al.~\cite{yao2022concept} proposed the Concept-level Model Interpretation framework (CMIC) to identify and rank concepts that contribute to machine learning model predictions by discovering their causal relationships. This aims to provide more comprehensive and understandable interpretations. Similarly, Matthew et al.\cite{moraffah2020causal} developed a method for generating causal post-hoc explanations for black-box classifiers. This approach leverages a learned low-dimensional data representation, a generative model, and information-theoretic causal influence measures, targeting both global and local explanations. It does this without necessitating labelled features or known causal structures and has been validated through controlled tests and practical image recognition tasks. 

Some efforts have been made to use some causal frameworks for XAI models and their prediction~\cite{moraffah2020causal}. The effort was made mostly in two directions: using the average treatment effect as a measure for feature attribution~\cite{goyal2019explaining,Feder-CausaLM-2021,abraham2022cebab} and using causal inference to enhance causal interpretability for Shapley values~\cite{pmlr-v162-jung22a,heskes2020causal,janzing2020feature}. Work in~\cite{generativeCausalExplanation2020} uses information flow in a causal DAG to capture the causal influence of a feature on the outcome. All the previously discussed causal interpretation methods work require a pre-specified causal structure, a causal graph, or the causal orders among variables. Works in~\cite{schwab2019granger,SchwabCXPlain2019} use the Granger causality principle~\cite{granger1969investigating} for feature attribution. Both methods build an explanation model for a machine learning model to estimate feature importance.  


Counterfactual explanations have emerged as a prominent method for interpretability in machine learning, aiding in the understanding of model decisions by hypothesizing changes to the input data and observing the corresponding changes in output~\cite{sokol2019counterfactual,akula2020cocox}. For instance, Sokol et al.~\cite{sokol2019counterfactual} provide illustrative examples of explanations, analyse their advantages and disadvantages, demonstrate their application in debugging the base model, and reveal the associated security and privacy concerns. A significant body of work in causal explanation focuses on counterfactual explanation~\cite{chou2022counterfactuals,guidotti2022counterfactual}. True causal counterfactual explanations need a causal graph and structural equations associated with it. Consequently, many counterfactual explanations are association-based. However, causal counterfactual explanations require robust foundational assumptions about the underlying causal mechanisms of the data~\cite{akula2020cocox}. This typically involves employing Structural Causal Models (SCMs)~\cite{pearl2009causal,pearl2018book}, comprising a causal graph and functions specifying the relationships between variables. Although this approach provides a more principled and potentially more accurate method of generating explanations, it introduces its own set of challenges.  Additionally, it is important to emphasise that counterfactual explanations fall under the category of local explanations, whereas our AMIE value is designed to provide global explanations, offering insights that apply across the entire dataset rather than to individual instances.

\end{document}